\documentclass[11pt]{article}

\usepackage[final]{acl}
\usepackage{amsmath}
\usepackage{amsfonts}       % blackboard math symbols
\usepackage{times}
\usepackage{latexsym}
\usepackage{natbib}
\usepackage{booktabs}
\usepackage{multirow}
\usepackage[T1]{fontenc}

\usepackage{amsthm}

\newtheorem{proposition}{Proposition}

\usepackage{subcaption}
\usepackage[utf8]{inputenc}

\usepackage{microtype}

\usepackage{inconsolata}

\usepackage{graphicx}

\title{Latent-LoRA: Compact Latent-Space Adapters with Gradient-Free Routing for Continual Learning}

\author{
  \textbf{Reza Rahimi Azghan\textsuperscript{1}},
  \textbf{Gautham Krishna Gudur\textsuperscript{2}},
  \textbf{Giulia Pedrielli\textsuperscript{3}},
\\
  \textbf{Pavan Turaga\textsuperscript{4}},
  \textbf{Hassan Ghasemzadeh\textsuperscript{1}}
\\
\\
  \textsuperscript{1}College of Health Solutions,
  Arizona State University, Phoenix, Arizona, USA
\\
  \textsuperscript{2}Department of Electrical and Computer Engineering,
  The University of Texas at Austin, Austin, Texas, USA
\\
  \textsuperscript{3}School of Computing and Augmented Intelligence,
  Arizona State University, Tempe, Arizona, USA
\\
  \textsuperscript{4}The GAME School,
  Arizona State University, Tempe, Arizona, USA
\\
  \small{
    \textbf{Correspondence:}
    \href{mailto:rrahimia@asu.edu}{rrahimia@asu.edu}
  }
}
\begin{document}
\maketitle
\begin{abstract}
    
Large language models generalize well to individual tasks but lack an inherent mechanism for learning them sequentially, leading to catastrophic forgetting. To mitigate this, LoRA-based continual learning methods allocate a separate low-rank adapter per task, yet existing approaches either require task identity at inference or sum all adapters indiscriminately, letting irrelevant branches distort the output. Recent gating-based solutions route inputs to the correct adapter but introduce trainable parameters that themselves need protection against forgetting. In this work, we observe that pooled token embeddings from a frozen LLM embedding layer already separate task distributions throughout the learning sequence. A Gaussian mixture model fitted on these embeddings, without any gradient-based training, is sufficient for task-agnostic adapter selection at test time. This eliminates the need for a learned gating module. On the adapter side, constraining each task's parameters to the principal subspace of the pretrained weights via SVD yields a compact latent-space parameterization. Within this subspace, orthogonal regularization directly controls inter-task interference. The resulting system, Latent-LoRA, is replay-free, requires no trainable routing component, and uses substantially fewer parameters per task. Experiments across five model scales and two established continual learning benchmarks show state-of-the-art performance with near-zero forgetting.\looseness=-1

\end{abstract}

\section{Introduction}

Continual learning (CL), which requires a model to learn a sequence of tasks without forgetting previously acquired knowledge, is a fundamental challenge for large language models (LLMs). While LLMs achieve strong performance across a wide range of tasks through pre-training and fine-tuning, they tend to overwrite earlier knowledge when trained on new tasks, a well-known phenomenon called catastrophic forgetting~\citep{mccloskey1989catastrophic, kemker2018measuring, french1999catastrophic}. This problem is particularly pronounced in LLMs due to their massive parameter counts, which give them high capacity for new tasks but little capacity for preserving old ones~\citep{shi2024continual, luo2023investigating, de2021continual, rebuffi2017icarl}.\looseness=-1

Low-rank adaptation (LoRA)~\citep{hu2021lora, dettmers2024qlora} has become a popular foundation for CL in LLMs. By reparameterizing weight updates as low-rank matrices, LoRA enables parameter-efficient fine-tuning that fits naturally into the CL setting~\citep{wang2024comprehensive}: each new task receives its own LoRA branch while previous branches remain frozen, preventing direct interference with earlier knowledge. Methods such as O-LoRA~\citep{wang2023olora} further constrain new branches to lie in subspaces orthogonal to previous ones to reduce cross-task overlap. However, a critical problem persists at inference. Since task identities are unavailable in the task-agnostic setting, O-LoRA integrates all branches by simple summation, forcing every branch to influence the output on every input. This indiscriminate aggregation allows branches trained for one task to distort predictions on another and reintroduces forgetting despite the frozen parameters and orthogonal training.\looseness=-1

Recent work has attempted to address this by inserting a learned gating module between the input and each branch. GainLoRA~\citep{gainlora}, for instance, introduces a per-task MLP gate trained to activate the corresponding branch while suppressing others. While effective, this approach has its own drawbacks. The number of gating parameters grows linearly with the number of tasks, adding significant overhead. More fundamentally, the gates themselves are susceptible to forgetting, since each new gate's training can interfere with the subspaces used by previous gates. To prevent this, GainLoRA applies Gradient Projection Memory~\citep{saha2021gradient} to constrain gate updates orthogonally. It essentially introduces a second continual learning mechanism to protect the first.\looseness=-1

In this work, we take a different approach. We observe that pooled token embeddings produced by an LLM's own frozen embedding layer already separate task distributions cleanly enough to be routed by a simple probabilistic model. Concretely, after each task, we fit a Gaussian mixture over that task's training-data embeddings; at inference, the resulting mixtures yield a posterior distribution over tasks for any input, which is used to weight the contribution of each task's adapter. The router has stored parameters (means and covariance) but no trainable ones, and cannot be corrupted by subsequent training. We pair this router with compact adapters based on LoRA-XS~\citep{benedek2024lora-xs}, which constrain each task's weight update to the principal subspace of the pretrained weights via SVD, reducing per-task storage to a small $r \times r$ matrix, where $r$ is the adapter rank, building on ideas explored in SVFT ~\citep{svft}. Each adapter is saved as an independent snapshot that is never modified after training, so no subsequent task can corrupt a previously learned adapter.\looseness=-1

The scientific contributions in this work are:

\begin{itemize}
    \item We propose a training-free Gaussian mixture router that operates on pooled token embeddings from the base LLM's frozen embedding layer. The router has no trainable parameters, requires no protection against forgetting, and adds negligible storage per task.\looseness=-1
    \item To the best of our knowledge, this is the first application of low-rank latent-space adapters to continual learning. By training only a small square matrix in the SVD-induced latent subspace of each pretrained weight, we substantially reduce per-task storage compared to the standard LoRA branches.\looseness=-1
    \item Our method, LatentLoRA, achieves state-of-the-art average performance on standard CL benchmarks in the task-agnostic, exemplar-free setting, without any trainable routing component.\looseness=-1
\end{itemize}

\section{Background}
\subsection{Continual Learning Setup}
Continual learning considers a setting where a model is trained on a sequence of tasks $\{\mathcal{D}_t\}_{t=1}^{T}$, arriving one at a time~\citep{parisi2019continual}. In the supervised setting, each task consists of input-label pairs $\mathcal{D}_t = \{(\boldsymbol{x}_{i}, y_{i})\}_{i=1}^{n_t}$ drawn from a task-specific distribution $P_t(\boldsymbol{x}, y)$, where $n_t$ is the number of examples in task $t$. The goal is to find parameters $\boldsymbol{\Theta}$ for a single model that performs well across all tasks seen so far:\looseness=-1
\begin{equation}
    \boldsymbol{\Theta}^{*}=\arg\max_{\boldsymbol{\Theta}} \sum_{t=1}^{T} \sum_{(\boldsymbol{x}, y) \in \mathcal{D}_t} \log \, p(y \mid \boldsymbol{x}; \boldsymbol{\Theta})
\end{equation}

The central challenge is that when training on task $t$, the model no longer has access to the data of previous tasks $\{\mathcal{D}_1, \ldots, \mathcal{D}_{t-1}\}$, which leads to catastrophic forgetting of earlier knowledge. In this work, we operate under the task-agnostic, exemplar-free setting~\citep{aljundi2019task}: the model trains on each task sequentially without storing any data from prior tasks, and at inference time, no task identity is provided.\looseness=-1

\subsection{Low-Rank Adaptation}

Low-Rank Adaptation (LoRA)~\citep{hu2021lora} is a parameter-efficient fine-tuning method~\citep{houlsby2019parameterefficienttransferlearningnlp} that exploits the observation that weight updates in pre-trained models tend to lie in a low-dimensional subspace~\citep{aghajanyan2021intrinsic}. For a pre-trained weight matrix $\boldsymbol{W} \in \mathbb{R}^{m \times n}$, LoRA constrains the update $\Delta \boldsymbol{W}$ to a low-rank decomposition $\Delta \boldsymbol{W} = \boldsymbol{B}\boldsymbol{A}$, where $\boldsymbol{B} \in \mathbb{R}^{m \times r}$, $\boldsymbol{A} \in \mathbb{R}^{r \times n}$, and $r \ll \min(m, n)$. The pre-trained weight $\boldsymbol{W}$ remains frozen during training, and the forward pass of the adapted layer becomes:\looseness=-1
\begin{equation}
    \boldsymbol{e} = (\boldsymbol{W} + \boldsymbol{B}\boldsymbol{A})\boldsymbol{h}
\end{equation}
where $\boldsymbol{h}$ and $\boldsymbol{e}$ denote the input and output of the layer, respectively. Only $\boldsymbol{B}$ and $\boldsymbol{A}$ are updated during fine-tuning to reduce the number of trainable parameters from $m \times n$ to $r \times (m + n)$.

\paragraph{LoRA in Continual Learning.} LoRA naturally lends itself to continual learning~\citep{biderman2024loralearnsforgets}: each new task can receive its own LoRA branch $(\boldsymbol{B}_t, \boldsymbol{A}_t)$ while all previous branches are frozen, preventing direct modification of earlier parameters. O-LoRA~\citep{wang2023olora} further constrains each new branch to lie in the orthogonal complement of previous branches to reduce subspace overlap across tasks. However, at inference time, since task identities are unavailable, O-LoRA aggregates all branches through addition:\looseness=-1
\begin{equation}
    \boldsymbol{e} = \left(\boldsymbol{W} + \sum_{t=1}^{T} \boldsymbol{B}_t \boldsymbol{A}_t\right) \boldsymbol{h}
\end{equation}

This forces every branch to influence every input and allows irrelevant adapters to degrade performance on earlier tasks. GainLoRA~\citep{gainlora} addresses this by introducing a learned gating module $g_t(\boldsymbol{x})$ per task that controls the contribution of each branch:\looseness=-1
\begin{equation}
    \boldsymbol{e} = \left(\boldsymbol{W} + \sum_{t=1}^{T} g_t(\boldsymbol{x}) \cdot \boldsymbol{B}_t \boldsymbol{A}_t\right) \boldsymbol{h}
\end{equation}

While effective, this approach introduces a per-task MLP gate with its own trainable parameters, and the gates themselves are susceptible to forgetting. GainLoRA mitigates this by applying Gradient Projection Memory (GPM) to the gate parameters, adding further complexity. Moreover, both the adapter and gate parameters scale linearly with the number of tasks.\looseness=-1
\section{Methodology}

\begin{figure*}
    \centering
    \includegraphics[width=0.9\linewidth]{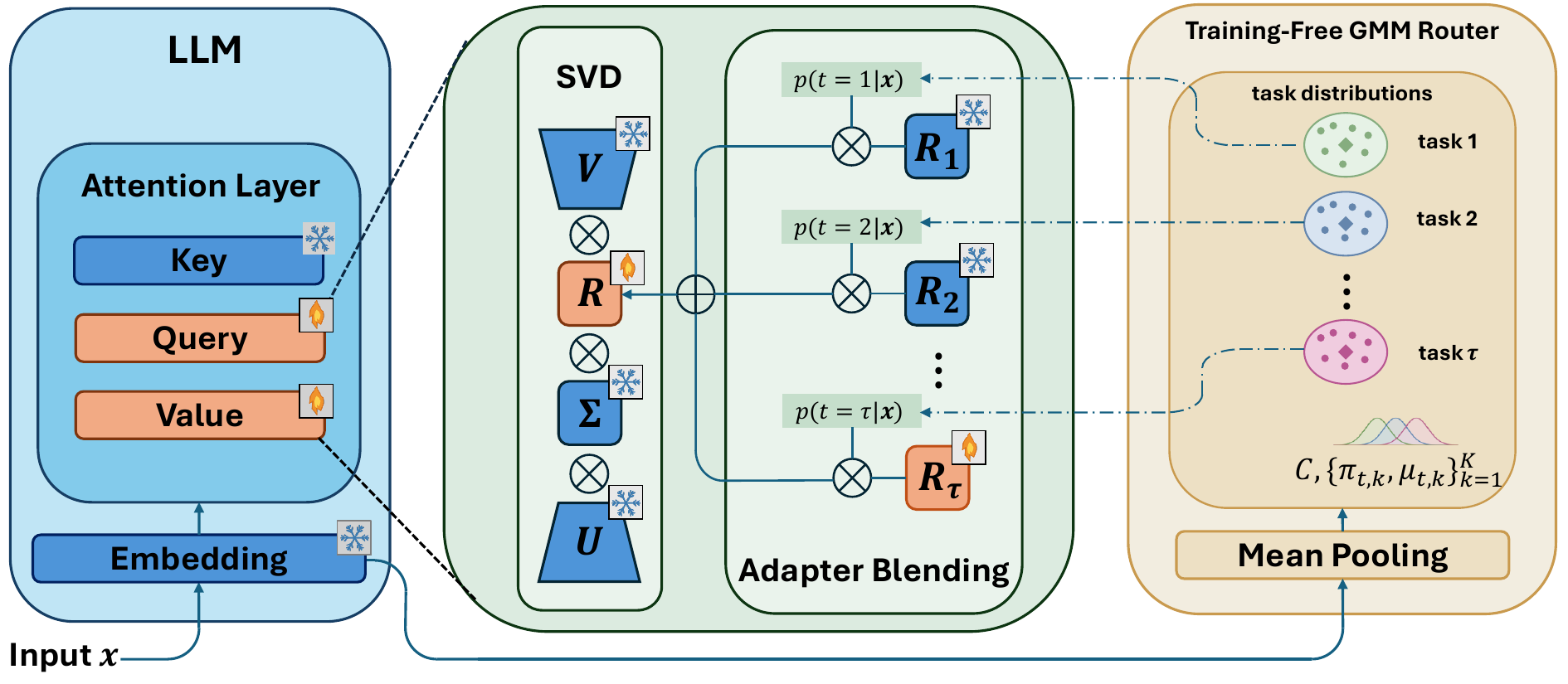}
    \caption{Latent-LoRA trains a per-task square matrix in the frozen SVD subspace and routes inputs via a training-free Gaussian mixture over pooled embeddings}
    \label{fig:system}
\end{figure*}

\subsection{Overview}
Figure~\ref{fig:system} illustrates the full pipeline of Latent-LoRA. The base language model remains frozen throughout, including its embedding layer. After training each task's adapter, we extract pooled token embeddings $\phi(\boldsymbol{x}) = \mathrm{MeanPool}(\mathrm{Emb}(\boldsymbol{x}))$ from the frozen embedding layer and fit a Gaussian mixture over them. The per-task mixtures collectively form a router that maps any input $\boldsymbol{x}$ to a posterior $p(t \mid \boldsymbol{x})$ over known tasks, without any trainable parameters (\S\ref{sec:router}). Each task's adapter $R_t$ is a compact matrix trained in the SVD-derived latent subspace of the pretrained weights and saved as an independent snapshot that no subsequent task can modify (\S\ref{sec:adapters}).\looseness=-1

\subsection{Compact Latent-Space Adapters}
\label{sec:adapters}

Standard LoRA-based continual learning methods allocate a pair of trainable low-rank matrices $(\boldsymbol{A}_t, \boldsymbol{B}_t)$ per task, with $r(m+n)$ parameters per target module. While modest compared to full fine-tuning, this cost scales linearly with the model dimension and accumulates across tasks. \looseness=-1

We adopt a more compact parameterization based on LoRA-XS~\citep{benedek2024lora-xs}, which constrains weight updates to the principal subspace of the pretrained weights. Given a target weight $\boldsymbol{W} \in \mathbb{R}^{m \times n}$, we compute its rank-$r$ truncated SVD $\boldsymbol{W} \approx \boldsymbol{U}_r \boldsymbol{\Sigma}_r \boldsymbol{V_r}^{\top}$ once and keep all three factors frozen. Here $\boldsymbol{U}_r \in \mathbb{R}^{m \times r}$ and $\boldsymbol{V}_r \in \mathbb{R}^{n \times r}$ have orthonormal columns, and $\boldsymbol{\Sigma}_r = \mathrm{diag}(\sigma_1, \ldots, \sigma_r)$ contains the top $r$ singular values. The per-task adapter is a small trainable matrix $\boldsymbol{R}_t \in \mathbb{R}^{r \times r}$, and the weight update is:
\begin{equation}
    \Delta \boldsymbol{W}_t = \boldsymbol{U}_r \, \boldsymbol{\Sigma}_r \, \boldsymbol{R}_t \, \boldsymbol{V}_r^{\top}.
    \label{eq:loraxs_update}
\end{equation}

This reduces trainable parameters per module from $r(m+n)$ to $r^2$, which is a factor of $(m+n)/r$ that grows with model size. \looseness=-1

By the Eckart-Young-Mirsky theorem~\citep{eckart1936approximation}, the subspace $\mathcal{S}_r = \{\boldsymbol{U}_r \boldsymbol{X} \boldsymbol{V}_r^{\top} : \boldsymbol{X} \in \mathbb{R}^{r \times r}\}$ is the optimal subspace for approximating gradient updates in the Frobenius norm, assuming that fine-tuning gradients lie close to the distribution of pretraining gradients. We refer the reader to \citep{benedek2024lora-xs} for a formal proof. Beyond parameter efficiency, this parameterization has a structural consequence for continual learning: because the high-dimensional factors $\boldsymbol{U}_r$ and $\boldsymbol{V}_r$ are frozen and orthonormal, the interference between any two task adapters reduces to a function of their $r \times r$  matrices alone.\looseness=-1

\paragraph{Output perturbation and interference.} Forgetting in LoRA-based CL arises when a new task's adapter perturbs the model's output on old-task inputs~\citep{goodfellow2013empirical}. We formalize this at the layer level. For an input $\boldsymbol{h}$, task $j$'s adapter produces the output perturbation:
\begin{equation}
\delta_j(\boldsymbol{h}) = \Delta \boldsymbol{W}_j \, \boldsymbol{h} = \boldsymbol{U}_r \boldsymbol{\Sigma}_r \boldsymbol{R}_j \boldsymbol{V}_r^{\top} \boldsymbol{h}  
\end{equation}

Define the \emph{latent-space projection} $\psi(\boldsymbol{h}) = \boldsymbol{V}_r^{\top} \boldsymbol{h} \in \mathbb{R}^r$ and the \emph{weighted adapter} $\tilde{\boldsymbol{R}}_t = \boldsymbol{\Sigma}_r \boldsymbol{R}_t \in \mathbb{R}^{r \times r}$. The inner product of the output perturbations from tasks $i$ and $j$ measures the extent to which task $j$'s perturbation has a component along the direction learned by task $i$, and it decomposes as:\looseness=-1
\begin{equation}
    \delta_i(\boldsymbol{h})^{\top} \delta_j(\boldsymbol{h})
    \;=\; \psi(\boldsymbol{h})^{\top} \; \tilde{\boldsymbol{R}}_i^{\top} \, \tilde{\boldsymbol{R}}_j \; \psi(\boldsymbol{h}),
    \label{eq:interference}
\end{equation}

The derivation is given in Appendix~\ref{app:proof}. This interference is a quadratic form in $\psi(\boldsymbol{h})$ whose kernel is $\tilde{\boldsymbol{R}}_i^{\top} \tilde{\boldsymbol{R}}_j$, a quantity that depends only on the adapters and singular values, with the high-dimensional projection factors $\boldsymbol{U}_r$ and $\boldsymbol{V}_r$ canceling out due to orthonormality.\looseness=-1

\paragraph{Orthogonal regularization.}
The decomposition~\eqref{eq:interference} identifies
$\tilde{\boldsymbol{R}}_i^{\top} \tilde{\boldsymbol{R}}_j$ as the quantity governing inter-task interference. We therefore regularize it directly:\looseness=-1
\begin{equation}
    \mathcal{L}_{\mathrm{ortho}}
    = \sum_{i < j}
    \bigl\|\tilde{\boldsymbol{R}}_i^{\top} \tilde{\boldsymbol{R}}_j\bigr\|_F^2.
    \label{eq:ortho_loss}
\end{equation}

This weighting arises naturally from the interference analysis: directions corresponding to larger singular values produce larger output perturbations and are penalized proportionally. We show that this regularization provides upper-bound control over the interference:\looseness=-1

\begin{proposition}
\label{prop:bound}
Under the latent-space adapter parameterization, the output interference
satisfies, for all inputs $\boldsymbol{h}$:\looseness=-1
\begin{equation}
    \bigl|\delta_i(\boldsymbol{h})^{\top} \delta_j(\boldsymbol{h})\bigr|
    \;\leq\; \bigl\|\tilde{\boldsymbol{R}}_i^{\top} \tilde{\boldsymbol{R}}_j\bigr\|_2
    \;\cdot\; \|\psi(\boldsymbol{h})\|^2.
    \label{eq:bound}
\end{equation}
\end{proposition}

\noindent

The proof derives from ~\eqref{eq:interference} and is given in Appendix~\ref{app:proof}. As the regularization~\eqref{eq:ortho_loss} drives $\|\tilde{\boldsymbol{R}}_i^{\top} \tilde{\boldsymbol{R}}_j\|_2$ towards zero, it progressively suppresses interference across all inputs. The factor $\|\psi(\boldsymbol{h})\|^2 = \|\boldsymbol{V}_r^{\top} \boldsymbol{h}\|^2$ depends only on the input and the pretrained model's singular subspace; it cannot grow during adapter training. Crucially, the regularization target is identical to the interference kernel, with no uncontrolled residual. We contrast this with standard LoRA's incomplete regularization and discuss the implicit regularization benefits of compact adapters in Appendix~\ref{app:analysis}.\looseness=-1

The full training objective for task $t$ combines the standard language modeling loss with the orthogonal regularization:\looseness=-1
\begin{equation}
    \mathcal{L}_t = -\sum_{(\boldsymbol{x}, \boldsymbol{y}) \in \mathcal{D}_t} \log p_\Theta(\boldsymbol{y} \mid \boldsymbol{x}) 
    \;+\; \lambda \sum_{i=1}^{t-1} \bigl\|\tilde{\boldsymbol{R}}_i^{\top} \tilde{\boldsymbol{R}}_t\bigr\|_F^2,
    \label{eq:total_loss}
\end{equation}
where $\lambda$ controls the strength of the orthogonal penalty. Only $\boldsymbol{R}_t$ receives gradients; all previous adapters $\{\boldsymbol{R}_i\}_{i<t}$ are frozen snapshots.\looseness=-1

\paragraph{Snapshot isolation.} 
After training on task $t$, the adapter $\boldsymbol{R}_t$ is saved as an independent snapshot and never modified. Once all adapters are frozen, performance on old tasks can only change through two channels: the router assigning a nonzero probability to the wrong adapter, or interference between adapters when they are blended via $\boldsymbol{R}(\boldsymbol{x}) = \sum_t p(t \mid \boldsymbol{x}) \boldsymbol{R}_t$. The orthogonal regularization controls the second; the training-free router, introduced next, controls the first.\looseness=-1

\begin{figure*}
    \centering
    \includegraphics[width=0.9\linewidth]{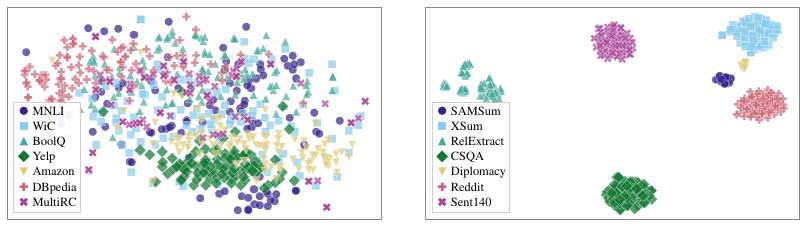}
    \caption{t-SNE projection of pooled token embeddings from the frozen T5-Large embedding layer. Left: Long Sequence. Right: SuperNI.\looseness=-1}
    \label{fig:tsne}
\end{figure*}

% \begin{figure}
%     \centering
%     \includegraphics[width=\linewidth]{figures/tsne_paper (7).pdf}
%     \caption{t-SNE projection of pooled token embeddings from the frozen T5-Large embedding layer. Left: Long Sequence. Right: SuperNI.}
%     \label{fig:tsne}
% \end{figure}

\subsection{Training-Free Task Router}
\label{sec:router}

In the task-agnostic setting, the model must determine which adapter(s) to apply without receiving a task identity. Prior methods either aggregate all adapters indiscriminately~\citep{wang2023olora} or introduce learned gating modules that require their own protection against forgetting~\citep{gainlora}. We take a different approach: we observe that the pooled token embeddings produced by the base model's frozen embedding layer already separate task distributions with sufficient margin for accurate routing, and exploit this with a simple probabilistic model that requires no gradient-based training.\looseness=-1

\paragraph{Embedding extraction.}
For an input text $\boldsymbol{x}$, we extract a fixed-size representation using the base model's own embedding layer:\looseness=-1
\begin{equation}
    \phi(\boldsymbol{x}) = \mathrm{MeanPool}\bigl(\mathrm{Emb}(\boldsymbol{x})\bigr),
    \label{eq:embed}
\end{equation}
where $\mathrm{Emb}$ denotes the model's input embedding table and $\mathrm{MeanPool}$ averages the resulting token-level vectors into a single vector $\phi(\boldsymbol{x}) \in \mathbb{R}^d$. Because LoRA adapters are applied to the attention projections and not to the embedding layer, $\phi(\boldsymbol{x})$ is stationary across the task sequence. Therefore, the distributions the router was fitted on do not drift as new tasks are learned.\looseness=-1

Figure~\ref{fig:tsne} visualizes $\phi(\boldsymbol{x})$ via t-SNE~\citep{maaten2008tsne} for tasks from two CL benchmarks, Long Sequence~\citep{razdaibiedina2023progressive} and SuperNI~\citep{wang2022super} using the frozen T5-Large~\citep{raffel2020t5} embedding layer. Task distributions form well-separated clusters without task-specific training, suggesting that pretrained embeddings provide useful structure for routing. Similar patterns across more language model scales are shown in Appendix~\ref{app:vis}.\looseness=-1

\paragraph{Gaussian mixture model.}
After training the adapter for task $t$, we fit a task-specific Gaussian mixture model (GMM) over the training embeddings $\{\phi(\boldsymbol{x}): \boldsymbol{x} \in \mathcal{D}_t\}$. Each task's distribution is modeled as a mixture of $K$ Gaussians~\citep{reynolds2009gaussian}:\looseness=-1
\begin{equation}
    p\bigl(\phi (\boldsymbol{x})\mid t\bigr) = \sum_{k=1}^{K} \pi_{t,k} \;
    \mathcal{N}\bigl(\phi (\boldsymbol{x}) \,\big|\, \mu_{t,k}, \mathbf{C}\bigr),
    \label{eq:gmm}
\end{equation}
where $\pi_{t,k}$ and $\mu_{t,k}$ are the weight and mean of component $k$ for task $t$, initialized via $K$-means clustering on the task's embeddings, and $\mathbf{C}$ is a covariance matrix shared across all tasks and components. Using multiple components per task allows the model to capture multi-modal task distributions; for instance, a question-answering task may contain both short factoid questions and long passage-based questions, which occupy different regions of the embedding space.\looseness=-1

The shared covariance is computed as the regularized pooled within-task scatter:\looseness=-1
\begin{equation}
    \mathbf{C} = \frac{\sum_{t=1}^{T} \mathbf{S}_t}{\sum_{t=1}^{T} n_t}
    + \epsilon \boldsymbol{I},
    \label{eq:cov}
\end{equation}
where $\mathbf{S}_t = \sum_{\boldsymbol{x} \in \mathcal{D}_t}(\phi(\boldsymbol{x}) - \bar\phi_t) (\phi(\boldsymbol{x}) - \bar\phi_t)^{\top}$ is the scatter matrix for task $t$, $n_t$ is the number of training samples, $\epsilon$ is a regularization coefficient for numerical stability, and $\bar\phi_t$ is the mean embedding for task $t$.\looseness=-1

% ── T5-Large main results table ──────────────────────────────────────
\begin{table*}[t]
\centering
\caption{Main results on T5-Large across both benchmarks and task 
orderings. AP: average performance (\%)$\uparrow$. FM: forgetting measure}
\label{tab:main_t5large}
\setlength{\tabcolsep}{5pt}
\begin{tabular}{l|cccccccc}
\toprule
 & \multicolumn{4}{c}{SuperNI} & \multicolumn{4}{c}{Long Sequence} \\
\cmidrule(lr){2-5} \cmidrule(lr){6-9}
 & \multicolumn{2}{c}{Order 1} & \multicolumn{2}{c}{Order 2} & \multicolumn{2}{c}{Order 3} & \multicolumn{2}{c}{Order 4} \\
% \cmidrule(lr){2-3} \cmidrule(lr){4-5} \cmidrule(lr){6-7} \cmidrule(lr){8-9}
Method & AP$\uparrow$ & FM$\downarrow$ & AP$\uparrow$ & FM$\downarrow$ & AP$\uparrow$ & FM$\downarrow$ & AP$\uparrow$ & FM$\downarrow$ \\
\midrule
LFPT5      & 39.03 & 10.88 & 29.71 & 20.72 & 66.62 & 14.51 & 67.40 & 13.11 \\
EWC        & 16.91 & 27.11 & 20.12 & 28.37 & 47.71 & 20.14 & 52.19 & 31.17 \\
LWF        & 19.34 & 24.74 & 30.34 & 17.63 & 50.73 & 19.17 & 47.94 & 33.22 \\
TaSL       & 25.72 & 19.92 & 28.09 & 18.20 & 71.37 & 6.20 & 72.91 & 6.03 \\
KIFLoRA    & 28.31 & 17.21 & 30.31 & 16.27 & 71.34 & 7.13 & 73.21 & 6.58 \\
SeqLoRA    & 8.64 & 48.41 & 7.17 & 47.22 & 48.86 & 29.78 & 30.46 & 44.97 \\
IncLoRA    & 13.08 & 42.26 & 16.39 & 36.51 & 62.75 & 12.75 & 61.76 & 15.35 \\
O-LoRA     & 27.91 & 17.23 & 33.35 & 11.11 & 71.31 & 3.41 & 71.60 & 4.16 \\
InfLoRA    & 39.02 & 8.02 & 39.97 & 8.22 & 75.02 & 3.01 & 75.35 & 2.39 \\
GainLoRA   & 46.77 & 2.14 & 47.25 & 2.12 & 78.21 & 0.72 & 76.26 & 1.14 \\
\midrule
Latent-LoRA     & \textbf{48.60} & \textbf{0.01} & \textbf{49.75} & \textbf{0.01} & \textbf{79.95} & \textbf{0.57} & \textbf{79.87} & \textbf{0.73} \\
\bottomrule
\end{tabular}
\end{table*}

Sharing a single covariance across tasks follows the Linear Discriminant Analysis (LDA) principle of pooled within-class scatter under a shared-covariance assumption, an approach also used in recent continual learning work~\citep{momeni2025klda,goswami2023fecam}. In our setting, this rescales the embedding space according to common variation across tasks, emphasizing task-discriminative directions. When a new task arrives, $\mathbf{C}$ is updated via~\eqref{eq:cov} without requiring access to previous tasks' data.\looseness=-1

\paragraph{Soft routing and adapter blending.}
At inference, the router computes the log-likelihood of $\phi(\boldsymbol{x})$ under each task's mixture and derives the task posterior under a uniform prior:\looseness=-1
\begin{equation}
    p(t \mid \boldsymbol{x}) = \frac{p(\phi(\boldsymbol{x}) \mid t)}
    {\sum_{t'=1}^{T} p(\phi(\boldsymbol{x}) \mid t')}.
    \label{eq:posterior}
\end{equation}

The posterior is used to blend all stored adapters into a single
effective adapter for input $\boldsymbol{x}$:\looseness=-1
\begin{equation}
    \boldsymbol{R}(\boldsymbol{x}) = \sum_{t=1}^{T} p(t \mid \boldsymbol{x}) \, \boldsymbol{R}_t.
    \label{eq:blend}
\end{equation}

When the posterior assigns a nonzero weight to a non-target adapter, the interference bound (Proposition~\ref{prop:bound}) constrains interference from non-target adapters. \looseness=-1

The router has two properties that distinguish it from learned alternatives. First, it has \emph{no trainable parameters}: the means and weights are set analytically from $K$-means, and the covariance is a pooled scatter statistic. No gradient ever flows into the router, so it does not introduce an additional source of forgetting and requires no auxiliary protection mechanism. Second, it is \emph{incremental}: adding a new task requires only fitting $K$ means on the new task's embeddings and updating the shared covariance, with no need to revisit or replay data from earlier tasks.\looseness=-1

\section{Experiments}
\label{sec:experiments}

\subsection{Setup}

\paragraph{Benchmarks.}
We evaluate on two continual learning benchmarks for LLMs. \textbf{SuperNI}~\citep{wang2022super} covers diverse NLP tasks including dialogue generation, information extraction, question answering, summarization, and sentiment analysis. Following~\citep{zhao2024sapt}, three tasks are selected from each category, yielding 15 tasks arranged into two orderings (Orders 1 and 2). \textbf{Long Sequence}~\citep{razdaibiedina2023progressive} consists of 15 classification tasks spanning natural language inference, sentiment analysis, topic classification, and question answering, also arranged into two orderings (Orders 3 and 4). We adopt the same task selections and orderings as GainLoRA~\citep{gainlora} and O-LoRA~\citep{wang2023olora} to ensure direct comparability. Full task lists are provided in Appendix~\ref{app:setup}.\looseness=-1

\paragraph{Metrics.}
We report two standard metrics. \textbf{Average Performance (AP)} is the mean score across all tasks after training on the final task, and \textbf{Forgetting Measure (FM)} measures the mean drop from each task's peak score to its final score:\looseness=-1
\begin{align}
    \mathrm{AP} &= \frac{1}{T}\sum_{j=1}^{T} a_{T,j}, \label{eq:ap} \\
    \mathrm{FM} &= \frac{1}{T{-}1}\sum_{j=1}^{T-1} \Bigl(\max_{i \geq j}\, a_{i,j} - a_{T,j}\Bigr), \label{eq:fm}
\end{align}
where $T$ is the total number of tasks and $a_{i,j}$ is the model's performance on task $j$ after training on task $i$.\looseness=-1

\paragraph{Baselines.}

We compare against a range of continual learning methods. These include regularization-based approaches such as EWC~\citep{kirkpatrick2017ewc} and LWF~\citep{li2017lwf}, the prompt-based method LFPT5~\citep{qin2022lfpt5}, and several LoRA-based methods: SeqLoRA, which sequentially fine-tunes a single adapter without forgetting mitigation; IncLoRA~\citep{wang2023olora}; O-LoRA~\citep{wang2023olora}; InfLoRA~\citep{liang2024inflora}; KIFLoRA~\citep{kiflora}; TaSL~\citep{tasl}; and GainLoRA~\citep{gainlora}. Following prior work~\citep{wang2023olora,gainlora}, we focus on the task-agnostic setting~\citep{zeng2019continuous} where task identities are unavailable at inference, and all methods operate without access to replay data from previous tasks.\looseness=-1

\paragraph{Implementation details.}
Latent-LoRA is model-agnostic and applicable to any transformer-based architecture. Following existing CL works~\citep{wang2023olora,gainlora,zhao2024sapt}, all methods use instruction tuning~\citep{wei2022finetuned} and are optimized with AdamW~\citep{loshchilov2019adamw}. To ensure fair comparisons, all LoRA-based baselines incorporate adapters into the query and value projections of each transformer block~\citep{vaswani2017attention} with rank $r{=}8$, following the protocol of prior work~\citep{wang2023olora,gainlora}. Our method uses rank $r{=}32$. We evaluate across both encoder-decoder~\citep{raffel2020t5} and decoder-only~\citep{touvron2023llama,dubey2024llama3} architectures at multiple scales. Each experiment is repeated three times with different seeds and the average is reported. Details on learning rates, batch sizes, adapter training, and GMM fitting are provided in Appendix~\ref{app:implementation}.\looseness=-1

\subsection{Main Results}
\label{sec:main_results}

Table~\ref{tab:main_t5large} presents the main results on T5-Large across both benchmarks and task orderings. Our method achieves the highest AP and the lowest FM on all four configurations. On SuperNI, where tasks span diverse NLP categories including generation and classification, we outperform the previous best method (GainLoRA) by 1.8--2.5 points in AP while reducing forgetting to near zero. On Long Sequence, we improve over GainLoRA by 1.7--3.6 points in AP with forgetting below 1\%.\looseness=-1

Notably, methods that sum adapters at inference, such as O-LoRA and InfLoRA, suffer from substantially higher forgetting, particularly on SuperNI, where task diversity makes adapter interference more severe. GainLoRA mitigates this through learned gating modules, but at the cost of additional trainable parameters and a GPM-based protection mechanism. Our method achieves stronger results without any learned routing component and with fewer parameters per task.\looseness=-1

Among non-LoRA baselines, LFPT5 performs competitively on SuperNI Order 1 but degrades on other configurations, while regularization-based methods (EWC, LWF) and sequential approaches (SeqLoRA) exhibit high forgetting across the board, suggesting that parameter-level regularization alone is insufficient for long task sequences.\looseness=-1

\subsection{Scaling to Larger Architectures}
\label{sec:scaling}

To assess whether our method generalizes beyond T5-Large, we evaluate on four additional model scales spanning both encoder-decoder and decoder-only architectures. Table~\ref{tab:t5_3b} reports results on T5-XLarge and Table~\ref{tab:llama} on three Llama variants. Latent-LoRA consistently outperforms baselines across all scales, with the performance gap widening on larger models. On Llama-2-13B, we surpass GainLoRA by over 3 points in AP while maintaining near-zero forgetting. \looseness=-1

% ── T5-3B table ──────────────────────────────────────────────────────
\begin{table}[t]
\centering
\caption{Results on T5-Xlarge. AP (\%)$\uparrow$ / FM (\%)$\downarrow$.}
\label{tab:t5_3b}
\small
\setlength{\tabcolsep}{4pt}
\resizebox{\columnwidth}{!}{%
\begin{tabular}{lcccc}
\toprule
 & \multicolumn{2}{c}{SuperNI} & \multicolumn{2}{c}{Long Sequence} \\
\cmidrule(lr){2-3} \cmidrule(lr){4-5}
Method & Order 1 & Order 2 & Order 3 & Order 4 \\
\midrule
O-LoRA & 37.91 / 10.12 & 41.22 / 6.36 & 74.98 / 1.63 & 76.63 / 3.10 \\
GainLoRA & 51.29 / 3.13 & 50.86 / 2.12 & 81.21 / \textbf{0.55} & 79.91 / \textbf{0.81} \\
Latent-LoRA & \textbf{52.64} / \textbf{0.01} & \textbf{53.71} / \textbf{0.01} & \textbf{81.61} / 0.74 & \textbf{80.39} / 1.00 \\
\bottomrule
\end{tabular}
}
\end{table}

\begin{table}[t]
\centering
\caption{Results on Llama models, SuperNI. AP (\%)$\uparrow$ / FM (\%)$\downarrow$.}
\label{tab:llama}
\small
\begin{tabular}{llcc}
\toprule
Model & Method & Order 1 & Order 2 \\
\midrule
& O-LoRA & 40.13 / 10.09 & 41.23 / 6.12 \\
& GainLoRA & 52.11 / 2.14 & 50.96 / 2.19 \\
\multirow{-3}{*}{Llama-2-7B}
  & LatentLoRA & \textbf{54.03} / \textbf{0.01} & \textbf{54.31} / \textbf{0.01} \\
\midrule
& O-LoRA & 42.91 / 9.19 & 42.33 / 5.71 \\
& GainLoRA & 54.17 / 1.91 & 53.47 / 2.10 \\
\multirow{-3}{*}{Llama-3-8B}
  & LatentLoRA & \textbf{56.09} / \textbf{0.01} & \textbf{56.33} / \textbf{0.01} \\
\midrule
& O-LoRA & 44.17 / 10.52 & 43.12 / 9.96 \\
& GainLoRA & 54.93 / 1.98 & 53.17 / 2.27 \\
\multirow{-3}{*}{Llama-2-13B}
  & LatentLoRA & \textbf{56.21} / \textbf{0.01} & \textbf{56.72} / \textbf{0.01} \\
\bottomrule
\end{tabular}
\end{table}

\begin{figure*}[t]
    \centering
    \includegraphics[width=0.9\linewidth]{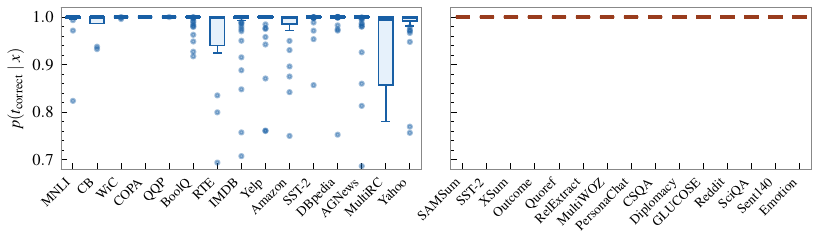}
    \caption{Router posterior probability of the correct task across test inputs for each task, evaluated on the embeddings of T5-Large. Left: Long Sequence: high confidence on most tasks with some variation. Right: SuperNI: near-perfect routing across all tasks.\looseness=-1}
    \label{fig:router_posterior}
\end{figure*}

\subsection{Ablation Study}
\label{sec:ablation}

We ablate our system on T5-Large to evaluate the compact adapter and the GMM router independently.\looseness=-1

\paragraph{Adapter.}
Table~\ref{tab:ablation_adapter} compares three configurations under 
sum-at-inference (no router): (i)~O-LoRA, using standard LoRA with orthogonal regularization on the down-projection matrices; 
(ii)~our compact adapter without orthogonal regularization 
($\lambda{=}0$) labeled as Latent-LoRA$^*$; and (iii)~our compact adapter with $\Sigma$-weighted orthogonal regularization (Eq.~\eqref{eq:ortho_loss}). Adding the $\Sigma$-weighted orthogonal penalty substantially improves the compact adapter, allowing it to surpass O-LoRA on most orderings while using far fewer trainable parameters.\looseness=-1

\begin{table}[t]
\centering
\caption{Ablation: adapter architecture under sum-at-inference (no router), T5-Large. AP$\uparrow$ / FM$\downarrow$.}
\label{tab:ablation_adapter}
% \small
\setlength{\tabcolsep}{4pt}
\resizebox{\columnwidth}{!}{%

\begin{tabular}{lcccc}
\toprule
 & \multicolumn{2}{c}{SuperNI} & \multicolumn{2}{c}{Long Sequence} \\
\cmidrule(lr){2-3} \cmidrule(lr){4-5}
Method & 1 & 2 & 3 & 4 \\
\midrule
O-LoRA & 27.91 / 17.23 & 33.35 / 11.11 & \textbf{71.31} / 3.41 & 71.60 / 4.16 \\
LatentLoRA$^*$ & 16.71 / 30.11 & 23.91 / 23.44 & 57.31 / 19.34 & 59.21 / 18.71 \\
LatentLoRA & \textbf{33.30} / \textbf{12.94} & \textbf{35.90} / \textbf{8.83} & 70.12 / \textbf{3.03} & \textbf{73.19}/ \textbf{3.97} \\
\bottomrule
\end{tabular}
}
\end{table}

\paragraph{Router.}
Figure~\ref{fig:router_posterior} evaluates the GMM router's accuracy. For each task, we compute the posterior probability assigned to the correct task, $p(t_{\mathrm{correct}} \mid \boldsymbol{x})$, across all test inputs. Each box plot summarizes this distribution over a fixed-size subsample. On SuperNI, the router assigns near-perfect probability to the correct task for nearly all inputs, consistent with the cluster separation in Figure~\ref{fig:tsne}. On Long Sequence, most tasks are routed with high confidence, though a few tasks show a wider spread. Results across all five model scales are reported in Appendix~\ref{app:vis}.\looseness=-1

\subsection{Parameter Efficiency}
\label{sec:efficiency}

Standard LoRA-based methods allocate $r(m{+}n)$ trainable parameters per target module, where $m$ and $n$ are the weight dimensions. With $r{=}8$, this amounts to 2.36M parameters per task on T5-Large and grows to 6.55M on Llama-2-13B. Our compact adapters require only $r^2$ parameters per module, a quantity independent of model dimension. With $r{=}32$, each task adds just 147K parameters on T5-Large and 82K on Llama-2-13B, a reduction of 16$\times$ and 80$\times$ respectively. Full comparison across all models is in Appendix~\ref{app:implementation} This gap widens with model scale, making the approach increasingly attractive for larger architectures. The GMM router adds negligible training cost: fitting $K$ means and updating a shared covariance takes seconds per task, with no gradient computation.\looseness=-1

During inference, routing requires only the embedding pass, which is already computed as part of the model's forward pass. This is followed by Mahalanobis distance evaluations against $K {\times} T$ Gaussian components and a softmax. This is a single vector operation on the pooled embedding, substantially cheaper than the per-task MLP forward passes required by learned gating approaches. The blended adapter $\boldsymbol{R}(\boldsymbol{x})$ is then inserted into each attention projection, adding the same overhead as a single LoRA forward pass.\looseness=-1
\section{Limitations}
\label{sec:limitations}

While our method achieves strong performance with minimal overhead, several limitations should be acknowledged. First, the shared covariance matrix $\mathbf{C} \in \mathbb{R}^{d \times d}$ must be inverted each time a new task is added, with $d$ being the embedding dimension of the base model. For models with large embedding dimensions (e.g., $d{=}4096$ for Llama-2-7B or $d{=}5120$ for Llama-2-13B), this inversion has $O(d^3)$ cost and requires careful numerical regularization. In practice, we find the Cholesky-based inversion with the regularization term $\epsilon \boldsymbol{I}$ to be stable across all models tested, but the cost grows cubically with embedding dimension and could become a concern for future models with substantially larger embedding spaces.

Second, the router relies on the assumption that task distributions are separable in the pooled embedding space. While this holds across all five model scales and both benchmarks in our experiments, tasks with highly overlapping input distributions (e.g., two sentiment analysis tasks differing only in domain) may not be reliably distinguished by the GMM. The router's soft blending provides graceful degradation in such cases, but the ceiling on routing accuracy is ultimately set by the separability of the frozen embeddings.

Finally, while we demonstrate strong results across five model scales up to 13B parameters and observed numbers improving with model size, we have not evaluated on models significantly larger than this. Whether the frozen embedding separability and compact adapter capacity hold at scales beyond 13B is an empirical question that warrants further investigation.
\bibliography{refs}
\newpage
\appendix
% \appendix

\section{Proofs}
\label{app:proof}

\subsection{Interference Decomposition}
\label{app:interference_proof}

We derive Eq.~\eqref{eq:interference} from the main text. Recall that each task's output perturbation is given by:\looseness=-1
\begin{equation*}
    \delta_j(\boldsymbol{h}) = \boldsymbol{U}_r \boldsymbol{\Sigma}_r \boldsymbol{R}_j \boldsymbol{V}_r^{\top} \boldsymbol{h}.
\end{equation*}

The inner product of the perturbations from tasks $i$ and $j$ is:\looseness=-1
\begin{align}
    \delta_i(\boldsymbol{h})^{\top} \delta_j(\boldsymbol{h})
    &= \bigl(\boldsymbol{U}_r \boldsymbol{\Sigma}_r \boldsymbol{R}_i \boldsymbol{V}_r^{\top} \boldsymbol{h}\bigr)^{\top}
       \bigl(\boldsymbol{U}_r \boldsymbol{\Sigma}_r \boldsymbol{R}_j \boldsymbol{V}_r^{\top} \boldsymbol{h}\bigr) \nonumber \\
    &= \boldsymbol{h}^{\top} \boldsymbol{V}_r \boldsymbol{R}_i^{\top} \boldsymbol{\Sigma}_r^{\top} \boldsymbol{U}_r^{\top}
       \boldsymbol{U}_r \boldsymbol{\Sigma}_r \boldsymbol{R}_j \boldsymbol{V}_r^{\top} \boldsymbol{h} \nonumber \\
    &= \boldsymbol{h}^{\top} \boldsymbol{V}_r \boldsymbol{R}_i^{\top} \boldsymbol{\Sigma}_r \boldsymbol{\Sigma}_r \boldsymbol{R}_j \boldsymbol{V}_r^{\top} \boldsymbol{h}
    \label{eq:step_uru}
\end{align}

where step~\eqref{eq:step_uru} uses $\boldsymbol{U}_r^{\top} \boldsymbol{U}_r = \boldsymbol{I}_r$ (since $\boldsymbol{U}_r$ has orthonormal columns) and $\boldsymbol{\Sigma}_r^{\top} = \boldsymbol{\Sigma}_r$ (since $\boldsymbol{\Sigma}_r$ is a diagonal matrix with real entries).\looseness=-1

Substituting the definitions $\psi(\boldsymbol{h}) = \boldsymbol{V}_r^{\top} \boldsymbol{h}$ and $\tilde{\boldsymbol{R}}_t = \boldsymbol{\Sigma}_r \boldsymbol{R}_t$:\looseness=-1
\begin{align}
    \delta_i(\boldsymbol{h})^{\top} \delta_j(\boldsymbol{h})
    &= \boldsymbol{h}^{\top} \boldsymbol{V}_r \boldsymbol{R}_i^{\top} \boldsymbol{\Sigma}_r \boldsymbol{\Sigma}_r \boldsymbol{R}_j \boldsymbol{V}_r^{\top} \boldsymbol{h} \nonumber \\
    &= \bigl(\boldsymbol{V}_r^{\top} \boldsymbol{h}\bigr)^{\top} \bigl(\boldsymbol{\Sigma}_r \boldsymbol{R}_i\bigr)^{\top} \bigl(\boldsymbol{\Sigma}_r \boldsymbol{R}_j\bigr) \bigl(\boldsymbol{V}_r^{\top} \boldsymbol{h}\bigr) \nonumber \\
    &= \psi(\boldsymbol{h})^{\top} \; \tilde{\boldsymbol{R}}_i^{\top} \, \tilde{\boldsymbol{R}}_j \; \psi(\boldsymbol{h}).
\end{align}

Note that the high-dimensional factors $\boldsymbol{U}_r \in \mathbb{R}^{m \times r}$ and $\boldsymbol{V}_r \in \mathbb{R}^{n \times r}$ cancel entirely: $\boldsymbol{U}_r$ vanishes through the orthonormality condition, and $\boldsymbol{V}_r$ is absorbed into the $r$-dimensional projection $\psi(\boldsymbol{h})$. The resulting expression depends only on the small $r \times r$ adapter matrices and the singular values.\looseness=-1

\subsection{Interference Bound}
\label{app:bound_proof}

We prove Proposition~\ref{prop:bound}: for all inputs $\boldsymbol{h}$,
\begin{equation*}
    \bigl|\delta_i(\boldsymbol{h})^{\top} \delta_j(\boldsymbol{h})\bigr|
    \;\leq\; \bigl\|\tilde{\boldsymbol{R}}_i^{\top} \tilde{\boldsymbol{R}}_j\bigr\|_2
    \;\cdot\; \|\psi(\boldsymbol{h})\|^2.
\end{equation*}

\begin{proof}
From the interference decomposition (Appendix~\ref{app:interference_proof}), we have,
\begin{equation*}
    \delta_i(\boldsymbol{h})^{\top} \delta_j(\boldsymbol{h})
    = \psi(\boldsymbol{h})^{\top} \boldsymbol{M} \, \psi(\boldsymbol{h}),
\end{equation*}
where $\boldsymbol{M} = \tilde{\boldsymbol{R}}_i^{\top} \tilde{\boldsymbol{R}}_j \in \mathbb{R}^{r \times r}$.\looseness=-1

For any matrix $\boldsymbol{M}$ and vector $\boldsymbol{u}$, the following standard inequality holds:\looseness=-1
\begin{equation*}
    \bigl|\boldsymbol{u}^{\top} \boldsymbol{M} \, \boldsymbol{u}\bigr|
    \;\leq\; \|\boldsymbol{M}\|_2 \, \|\boldsymbol{u}\|^2,
\end{equation*}
where $\|\boldsymbol{M}\|_2 = \sup_{\|\boldsymbol{v}\|=1} \|\boldsymbol{M}\boldsymbol{v}\|$ is the spectral norm (largest singular value) of $\boldsymbol{M}$. This follows from Cauchy--Schwarz:\looseness=-1
\begin{align}
    \bigl|\boldsymbol{u}^{\top} \boldsymbol{M} \, \boldsymbol{u}\bigr|
    &\leq \|\boldsymbol{u}\| \cdot \|\boldsymbol{M} \, \boldsymbol{u}\| \nonumber \\
    &\leq \|\boldsymbol{u}\| \cdot \|\boldsymbol{M}\|_2 \, \|\boldsymbol{u}\| \nonumber \\
    &= \|\boldsymbol{M}\|_2 \, \|\boldsymbol{u}\|^2.
\end{align}

Applying this with $\boldsymbol{u} = \psi(\boldsymbol{h})$ and $\boldsymbol{M} = \tilde{\boldsymbol{R}}_i^{\top} \tilde{\boldsymbol{R}}_j$:
\begin{equation*}
    \bigl|\delta_i(\boldsymbol{h})^{\top} \delta_j(\boldsymbol{h})\bigr|
    \;\leq\; \bigl\|\tilde{\boldsymbol{R}}_i^{\top} \tilde{\boldsymbol{R}}_j\bigr\|_2
    \;\cdot\; \|\psi(\boldsymbol{h})\|^2.
\end{equation*}

Since $\psi(\boldsymbol{h}) = \boldsymbol{V}_r^{\top} \boldsymbol{h}$ and $\boldsymbol{V}_r$ is fixed from the pretrained SVD, the factor $\|\psi(\boldsymbol{h})\|^2$ depends only on the input and the pretrained model's right singular vectors. It is independent of the adapter parameters and cannot be increased by training. Therefore, as the regularization loss $\mathcal{L}_{\mathrm{ortho}}$ drives $\|\tilde{\boldsymbol{R}}_i^{\top} \tilde{\boldsymbol{R}}_j\|_2 \to 0$, the interference $|\delta_i(\boldsymbol{h})^{\top} \delta_j(\boldsymbol{h})|$ is driven to zero uniformly over all inputs $\boldsymbol{h}$.\looseness=-1
\end{proof}
\section{Compact Adapters vs. Standard LoRA}
\label{app:analysis}

This appendix provides additional analysis of the compact latent-space adapter introduced in Section~\ref{sec:adapters}, contrasting its regularization properties with those of standard LoRA and discussing the role of parameter capacity in continual learning.\looseness=-1

\subsection{Incomplete Regularization in Standard LoRA}

In standard LoRA, each task's weight update takes the form $\Delta \boldsymbol{W}_t = \boldsymbol{B}_t \boldsymbol{A}_t$, where both $\boldsymbol{A}_t$ and $\boldsymbol{B}_t$ are trainable. The interference between two tasks' perturbations is,\looseness=-1
\begin{equation}
    \delta_i(\boldsymbol{h})^{\top} \delta_j(\boldsymbol{h}) 
    = \boldsymbol{h}^{\top} \boldsymbol{A}_i^{\top} \boldsymbol{B}_i^{\top} 
    \boldsymbol{B}_j \boldsymbol{A}_j \boldsymbol{h}.
\end{equation}

Existing continual learning methods such as O-LoRA~\citep{wang2023olora} penalize $\|\boldsymbol{A}_i^{\top} \boldsymbol{A}_j\|$ to encourage orthogonality between the down-projection subspaces. However, the up-projections $\boldsymbol{B}_i, \boldsymbol{B}_j$ remain unconstrained. Because these factors enter the interference multiplicatively, even perfect orthogonality in $\boldsymbol{A}$ does not eliminate interference: the $\boldsymbol{B}_i^{\top} \boldsymbol{B}_j$ term can still take arbitrary values. No uniform bound over all inputs analogous to Proposition~\ref{prop:bound} can be established for standard LoRA.\looseness=-1

In the compact parameterization of Section~\ref{sec:adapters}, the input-facing projection ($\boldsymbol{\Sigma}_r \boldsymbol{V}_r^{\top}$, analogous to $\boldsymbol{A}$) and the output-facing projection ($\boldsymbol{U}_r$, analogous to $\boldsymbol{B}$) are both frozen from the pretrained SVD. The trainable $\boldsymbol{R}_t$ is the sole degree of freedom, allowing the bound in Proposition~\ref{prop:bound} to cover the complete interference expression.\looseness=-1

\subsection{Compact Capacity as Implicit Regularization}

Each compact adapter operates with $r^2$ free parameters within a fixed subspace, compared to $r(m+n)$ in standard LoRA. Recent work suggests that restricting the trainable parameter space reduces forgetting: \citet{biderman2024loralearnsforgets} attribute LoRA's lower forgetting relative to full fine-tuning to its reduced capacity, \citet{aghajanyan2021intrinsic} show that effective fine-tuning occurs in a low-dimensional intrinsic subspace, and \citet{azghan2026gatedadaptationcontinuallearning} observe that structured parameter-efficient tuning improves cross-task consistency. Our parameterization pushes this further by confining each adapter to an $r^2$-dimensional space anchored to the pretrained principal subspace. This benefit is most apparent when combined with the router, which limits the number of adapters contributing to each input.\looseness=-1
\begin{table*}[htbp]
\centering
\caption{Tasks in the Long Sequence benchmark.}
\label{tab:long_tasks}
\begin{tabular}{llll}
\toprule
Task & Category & Domain & Metric \\
\midrule
Yelp & Sentiment analysis & Yelp reviews & Accuracy \\
Amazon & Sentiment analysis & Amazon reviews & Accuracy \\
IMDB & Sentiment analysis & Movie reviews & Accuracy \\
SST-2 & Sentiment analysis & Movie reviews & Accuracy \\
DBpedia & Topic classification & Wikipedia & Accuracy \\
AG News & Topic classification & News & Accuracy \\
Yahoo & Topic classification & Yahoo Q\&A & Accuracy \\
MNLI & Natural language inference & Various & Accuracy \\
QQP & Paraphrase detection & Quora & Accuracy \\
RTE & Natural language inference & News, Wikipedia & Accuracy \\
CB & Natural language inference & Various & Accuracy \\
WiC & Word sense disambiguation & Lexical databases & Accuracy \\
COPA & Question answering & Blogs, encyclopedia & Accuracy \\
BoolQ & Boolean question answering & Wikipedia & Accuracy \\
MultiRC & Question answering & Various & Accuracy \\
\bottomrule
\end{tabular}
\end{table*}

\section{Experimental Setup}
\label{app:setup}

\subsection{Benchmark Details}

\paragraph{Long Sequence Benchmark.}
The Long Sequence benchmark~\citep{razdaibiedina2023progressive} consists of 15 classification tasks drawn from established NLP datasets. Table~\ref{tab:long_tasks} summarizes each task's category, domain, and evaluation metric. All tasks are evaluated using accuracy.\looseness=-1

\paragraph{SuperNI Benchmark.}
The SuperNI benchmark~\citep{wang2022super} consists of 15 tasks spanning five NLP categories: dialogue generation, information extraction, question answering, summarization, and sentiment analysis. Following~\citet{zhao2024sapt}, three tasks are selected from each category. Table~\ref{tab:superni_tasks} lists the tasks and their evaluation metrics. Classification tasks are evaluated using accuracy; all others use Rouge-L.\looseness=-1

\begin{table*}[htbp]
\centering
\caption{Tasks in the SuperNI benchmark.}
\label{tab:superni_tasks}
\begin{tabular}{lll}
\toprule
Task & Category & Metric \\
\midrule
Task639 (MultiWOZ) & Dialogue generation & Rouge-L \\
Task1590 (Diplomacy) & Dialogue generation & Rouge-L \\
Task1729 (PersonaChat) & Dialogue generation & Rouge-L \\
Task181 (Outcome extraction) & Information extraction & Rouge-L \\
Task748 (GLUCOSE) & Information extraction & Rouge-L \\
Task1510 (Relation extraction) & Information extraction & Rouge-L \\
Task002 (Quoref) & Question answering & Rouge-L \\
Task073 (CommonsenseQA) & Question answering & Rouge-L \\
Task591 (SciQA) & Question answering & Rouge-L \\
Task511 (Reddit TIFU) & Summarization & Rouge-L \\
Task1290 (XSum) & Summarization & Rouge-L \\
Task1572 (SAMSum) & Summarization & Rouge-L \\
Task363 (SST-2) & Sentiment analysis & Accuracy \\
Task875 (Emotion) & Sentiment analysis & Accuracy \\
Task1687 (Sentiment140) & Sentiment analysis & Accuracy \\
\bottomrule
\end{tabular}
\end{table*}

\subsection{Task Orderings}

Table~\ref{tab:task_orders} lists the task sequences used in all experiments. Orders 1 and 2 correspond to the SuperNI benchmark; Orders 3 and 4 correspond to the Long Sequence benchmark. These orderings follow those used by O-LoRA~\citep{wang2023olora} and GainLoRA~\citep{gainlora}.\looseness=-1

\begin{table*}[htbp]
\centering
\caption{Task orderings used in all experiments.}
\label{tab:task_orders}
\begin{tabular}{cl}
\toprule
Order & Task sequence \\
\midrule
\multirow{3}{*}{1 (SuperNI)} 
  & Task1572 $\to$ Task363 $\to$ Task1290 $\to$ Task181 $\to$ Task002 $\to$ \\
  & Task1510 $\to$ Task639 $\to$ Task1729 $\to$ Task073 $\to$ Task1590 $\to$ \\
  & Task748 $\to$ Task511 $\to$ Task591 $\to$ Task1687 $\to$ Task875 \\
\midrule
\multirow{3}{*}{2 (SuperNI)} 
  & Task748 $\to$ Task073 $\to$ Task1590 $\to$ Task639 $\to$ Task1572 $\to$ \\
  & Task1687 $\to$ Task591 $\to$ Task363 $\to$ Task1510 $\to$ Task1729 $\to$ \\
  & Task181 $\to$ Task511 $\to$ Task002 $\to$ Task1290 $\to$ Task875 \\
\midrule
\multirow{3}{*}{3 (Long Seq.)} 
  & MNLI $\to$ CB $\to$ WiC $\to$ COPA $\to$ QQP $\to$ \\
  & BoolQ $\to$ RTE $\to$ IMDB $\to$ Yelp $\to$ Amazon $\to$ \\
  & SST-2 $\to$ DBpedia $\to$ AG News $\to$ MultiRC $\to$ Yahoo \\
\midrule
\multirow{3}{*}{4 (Long Seq.)} 
  & Yelp $\to$ Amazon $\to$ MNLI $\to$ CB $\to$ COPA $\to$ \\
  & QQP $\to$ RTE $\to$ IMDB $\to$ SST-2 $\to$ DBpedia $\to$ \\
  & AG News $\to$ Yahoo $\to$ MultiRC $\to$ BoolQ $\to$ WiC \\
\bottomrule
\end{tabular}
\end{table*}

\subsection{Implementation Details}
\label{app:implementation}

\paragraph{Model architectures.}
We evaluate on five model scales: T5-Large (770M parameters), T5-XLarge (3B), Llama-2-7B, Llama-3-8B, and Llama-2-13B. For all models, LoRA adapters are applied to the query and value projections of each attention layer. T5-Large contains 144 target modules (24 encoder self-attention + 24 decoder self-attention + 24 decoder cross-attention, each with query and value projections). T5-XL has the same architecture with larger projection dimensions ($4096 \times 1024$). Llama models contain 64 target modules for 7B/8B (32 layers $\times$ 2) and 80 for 13B (40 layers $\times$ 2).\looseness=-1

\paragraph{Adapter configuration.}
All LoRA-based baselines use rank $r{=}8$ and $\alpha{=}16$. Our method uses rank $r{=}32$ and $\alpha{=}16$. The per-task trainable parameter counts are summarized in Table~\ref{tab:param_details}.\looseness=-1

\begin{table}[htbp]
\centering
\caption{Per-task trainable parameters across models.}
\label{tab:param_details}
\begin{tabular}{lccc}
\toprule
Model & LoRA ($r{=}8$) & Ours ($r{=}32$) & Reduction \\
\midrule
T5-Large & 2,359K & 147K & 16$\times$ \\
T5-XL & 5,898K & 147K & 40$\times$ \\
Llama-2-7B & 4,194K & 66K & 64$\times$ \\
Llama-3-8B & 4,194K & 66K & 64$\times$ \\
Llama-2-13B & 6,554K & 82K & 80$\times$ \\
\bottomrule
\end{tabular}
\end{table}

\paragraph{Training.}
All methods are optimized with AdamW~\citep{loshchilov2019adamw}. For T5 models, we use a learning rate of $3 \times 10^{-4}$ and a batch size of 8. For Llama models, we use a learning rate of $5 \times 10^{-5}$ and a batch size of 4. All experiments use a constant learning rate schedule. Each task is trained for 30 epochs on T5 and 15 epochs on Llama. The orthogonal regularization coefficient is set to $\lambda = 0.05$ for SuperNI and $\lambda = 0.02$ for Long Sequence, selected based on the relative magnitude of the $\Sigma$-weighted ortho loss to the task loss on a held-out validation split of T5-Large.\looseness=-1

\paragraph{GMM router.}
The router uses $K{=}5$ Gaussian components per task, initialized via $K$-means with 30 iterations. The shared covariance regularization coefficient is $\epsilon = 0.01$. At inference, we use soft routing (Eq.~\ref{eq:blend}). The GMM fitting step takes approximately 2--5 seconds per task on CPU and does not require GPU resources.\looseness=-1

\paragraph{Evaluation.}
For classification tasks, we report exact-match accuracy. For generation tasks (SuperNI), we report Rouge-L~\citep{lin2004rouge} computed using the \texttt{rouge\_score} library with stemming enabled. During evaluation, outputs are generated using greedy decoding with a maximum generation length of 50 tokens.\looseness=-1

\paragraph{Hardware.}
All experiments are conducted on NVIDIA A100 GPUs. T5-Large and T5-XL experiments run on a single GPU. Llama experiments use a single GPU with gradient checkpointing enabled. Each experiment is repeated three times with different random seeds, and the average is reported.\looseness=-1
\section{Additional Visualizations}
\label{app:vis}

Figures~\ref{fig:tsne_appendix} and~\ref{fig:router_appendix} extend the t-SNE and router posterior analyses from the main text to three additional model scales. The patterns observed on T5-Large (Figures~\ref{fig:tsne} and~\ref{fig:router_posterior}) hold consistently: SuperNI tasks form clearly separable clusters across all models, and Long Sequence tasks show adequate margins with occasional overlap among semantically similar tasks. Router posteriors remain near-perfect on SuperNI and high-confidence on Long Sequence regardless of model family or scale.\looseness=-1

\begin{figure*}[t]
    \centering
    \begin{subfigure}{\linewidth}
        \centering
        \includegraphics[width=\linewidth]{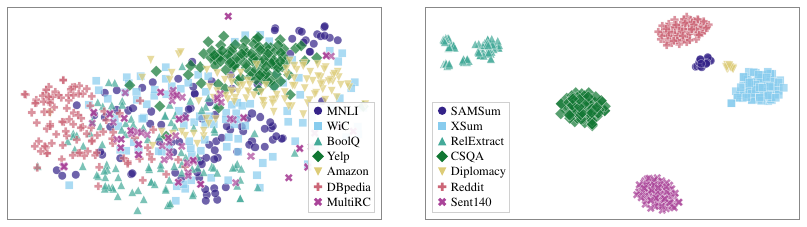}
        \caption{T5-XL}
    \end{subfigure}
    \vspace{0.3em}
    \begin{subfigure}{\linewidth}
        \centering
        \includegraphics[width=\linewidth]{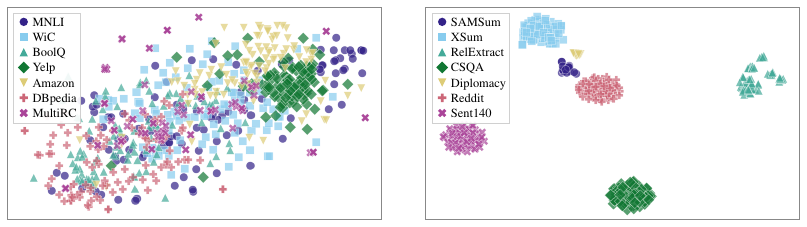}
        \caption{Llama-2-7B}
    \end{subfigure}
    \vspace{0.3em}
    \begin{subfigure}{\linewidth}
        \centering
        \includegraphics[width=\linewidth]{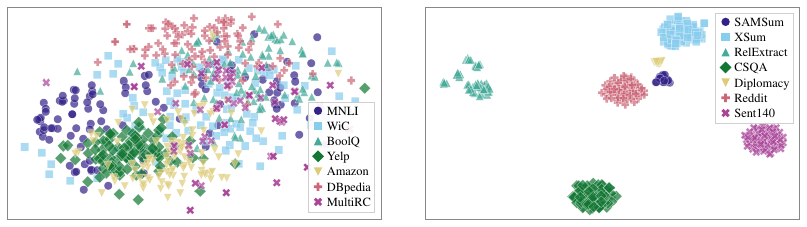}
        \caption{Llama-3-8B}
    \end{subfigure}
    \caption{t-SNE projections of pooled token embeddings from frozen 
    embedding layers. Left: Long Sequence. Right: SuperNI.}
    \label{fig:tsne_appendix}
\end{figure*}

\begin{figure*}[t]
    \centering
    \begin{subfigure}{\linewidth}
        \centering
        \includegraphics[width=\linewidth]{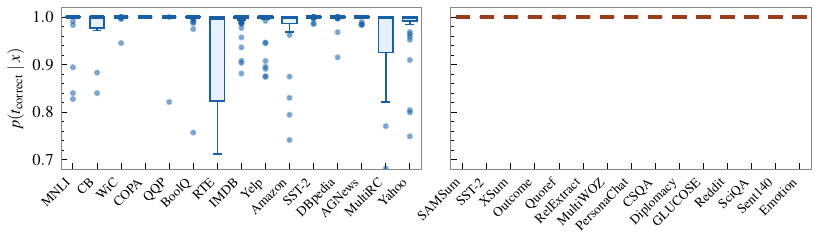}
        \caption{T5-XL}
    \end{subfigure}
    \vspace{0.3em}
    \begin{subfigure}{\linewidth}
        \centering
        \includegraphics[width=\linewidth]{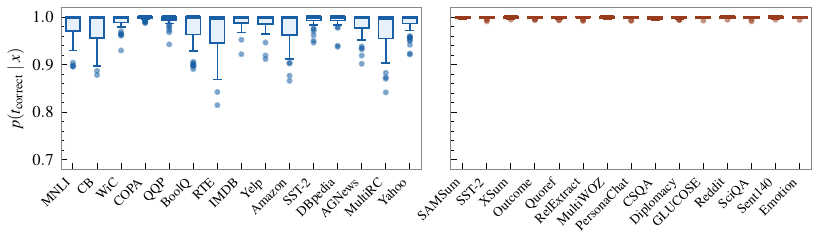}
        \caption{Llama-2-7B}
    \end{subfigure}
    \vspace{0.3em}
    \begin{subfigure}{\linewidth}
        \centering
        \includegraphics[width=\linewidth]{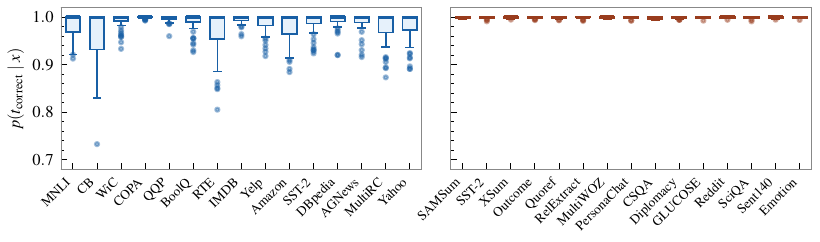}
        \caption{Llama-3-8B}
    \end{subfigure}
    \caption{Router posterior probability of the correct task across 
    test inputs. Left: Long Sequence. Right: SuperNI.}
    \label{fig:router_appendix}
\end{figure*}
\end{document}